\documentclass[twoside]{article}

\usepackage[accepted]{style}

\usepackage[style=authoryear,uniquename=false,maxcitenames=3,maxbibnames=10,sorting=nyvt]{biblatex}
\usepackage{rotating}
\usepackage{adjustbox}

\usepackage{graphicx}
\usepackage{tikz}
\usepackage{float}
\usepackage{subcaption}

\usepackage{algpseudocode}

\usepackage{amsmath}
\usepackage{amssymb}
\usepackage{amsthm}
\usepackage{bm}
\usepackage{bbm}

\usepackage{chngcntr}

\theoremstyle{plain}
\newtheorem{theorem}{Theorem}[section]

\newtheorem{corollary}[theorem]{Corollary}
\newtheorem{proposition}[theorem]{Proposition}

\theoremstyle{definition}

\theoremstyle{remark}

\numberwithin{equation}{section}

\let\c\mathcal
\let\bb\mathbb

\DeclareMathOperator*{\logsumexp}{logsumexp}
\DeclareMathOperator*{\expectation}{\bb E}

\def\given{\;|\;}
\def\d{\;\mathrm{d}}

\newcommand{\set}[1]{\left\{#1\right\}}

\newcounter{eqcounter}[enumi]
\newcommand{\eq}{\stepcounter{eqcounter}\overset{\text{(\alph{eqcounter})}}}

\def\ito{It\^{o}}

\begin{document}

\runningauthor{Boserup, Yang, Severinsen, Hipsley \& Sommer}

\twocolumn[

\styletitle{Parameter Inference via Differentiable Diffusion Bridge Importance Sampling}

\styleauthor{
    Nicklas Boserup\\
    Department of Computer Science\\
    University of Copenhagen
    \And
    Gefan Yang\\
    Department of Computer Science\\
    University of Copenhagen
    \And
    Michael Lind Severinsen\\
    Globe Institute\\
    University of Copenhagen
    \AND
    Christy Anna Hipsley\\
    Department of Biology\\
    University of Copenhagen
    \And
    Stefan Sommer\\
    Department of Computer Science\\
    University of Copenhagen
}

\styleaddress{}

\vspace*{1em}
]

\begin{abstract}
We introduce a methodology for performing parameter inference in high-dimensional, non-linear diffusion processes. We illustrate its applicability for obtaining insights into the evolution of and relationships between species, including ancestral state reconstruction.
Estimation is performed by utilising score matching to approximate diffusion bridges, which are subsequently used in an importance sampler to estimate log-likelihoods. The entire setup is differentiable, allowing gradient ascent on approximated log-likelihoods. This allows both parameter inference and diffusion mean estimation.
This novel, numerically stable, score matching-based parameter inference framework is presented and demonstrated on biological two- and three-dimensional morphometry data.
\end{abstract}

\section{INTRODUCTION}
Parameter inference from low-frequency observations in models involving non-linear stochastic differential equations with hundreds of correlated dimensions is inherently difficult due to lack of closed-form likelihoods and because of ill-conditioned numerics. In this paper, we combine deep learning-based score matching with statistical methods for parameter estimation in diffusion models to enable this. Specifically, we propose
\begin{itemize}
    \item a novel, numerically stable objective function for deep learning-based score matching, enabling direct simulation of diffusion bridges in hundreds of correlated dimensions,
    \item a fully differentiable likelihood estimator, utilising simulated diffusion bridges as proposals in a sample-efficient importance sampler, allowing parameter inference as well as diffusion mean estimation, and
    \item a series of techniques for circumventing numerical instability issues, enabling parameter inference via multivariate Gaussian approximations with neither determinant calculations nor matrix inversions.
\end{itemize}

The proposed methodology is demonstrated on problems in evolutionary biology, where non-linear and high-dimensional processes describing evolving shapes of species occur naturally.
This allows modelling morphological trait variation among species; e.g.\ establishing the most likely process from an unknown common ancestor given observations of extant species phenotypes, such as the landmark wing outlines of the two butterflies in Figure~\ref{fig:butterfly_landmarks}.

\begin{figure}[ht]
    \centering
    \includegraphics[width=0.4\linewidth]{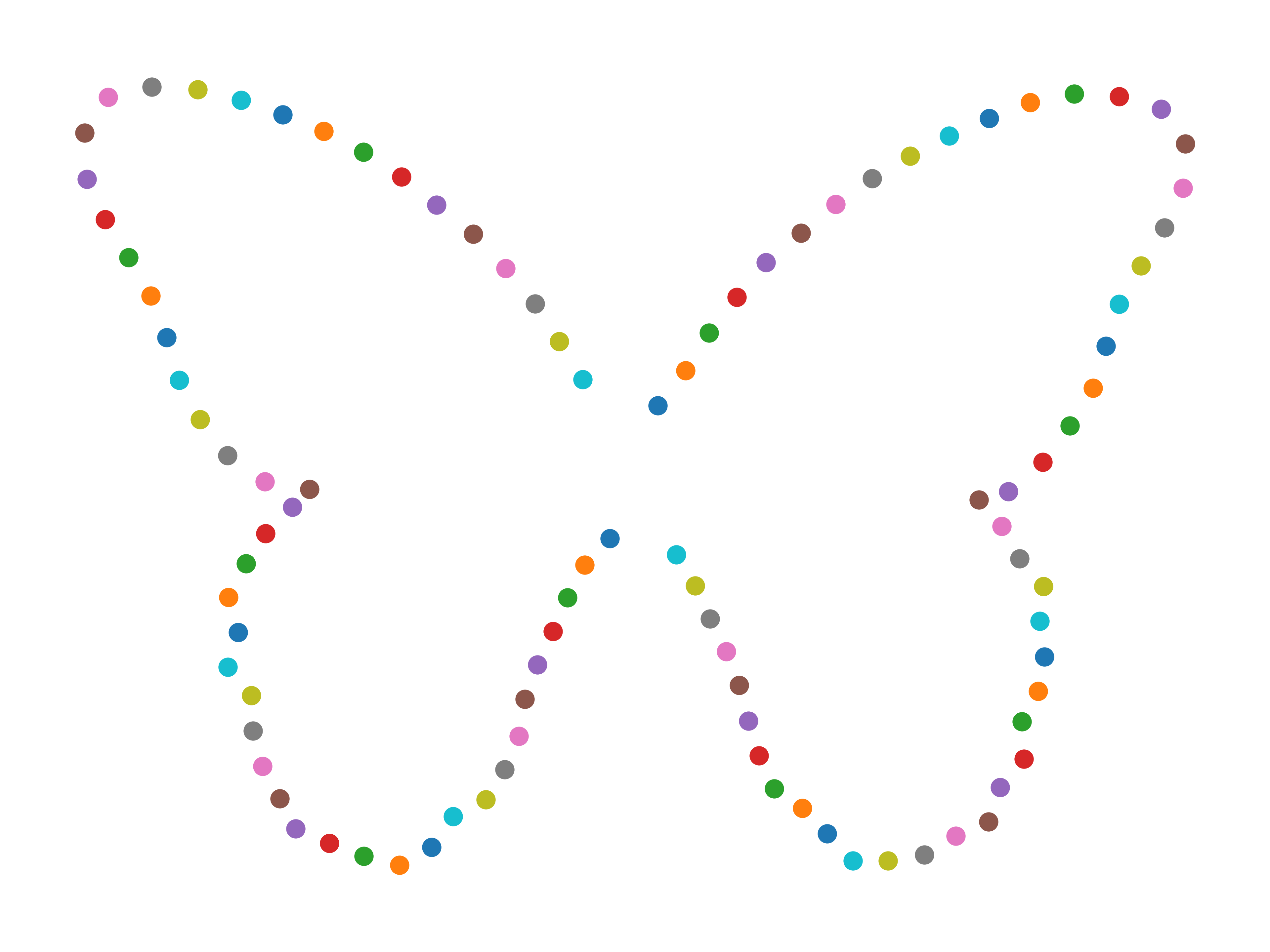}%
    \includegraphics[width=0.4\linewidth]{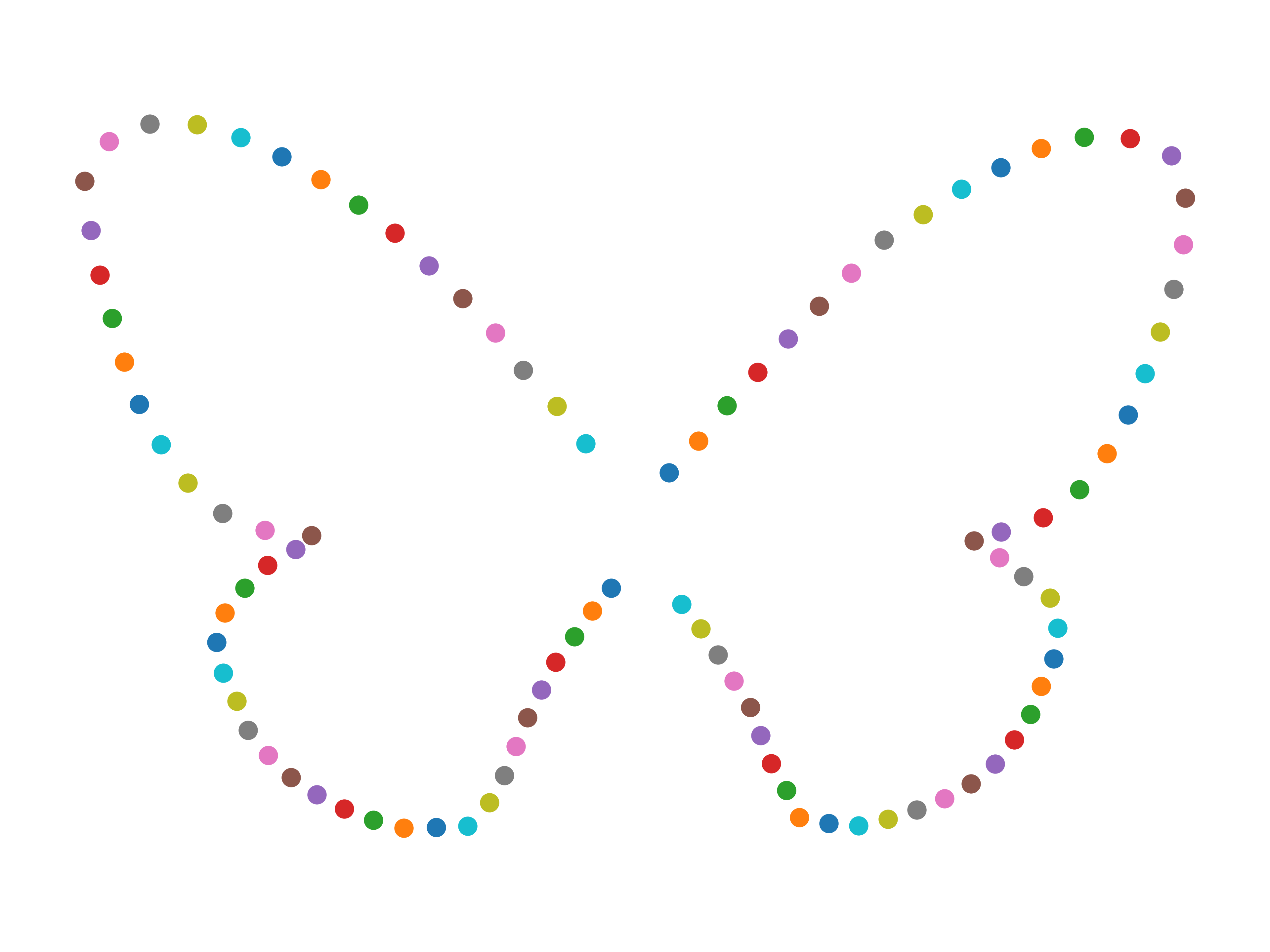}%
    \caption{100 landmarks describing the wing outlines of \emph{Papilio ambrax} and \emph{Papilio slateri}, respectively. Data from \textcite{gbif_m}.}
    \label{fig:butterfly_landmarks}
\end{figure}

Implementation of proposed methodology available at \url{https://github.com/nickeopti/msc-thesis-code}.

\section{STOCHASTIC MORPHOMETRY}
Numerous works, including \textcite{sommer2017bridge}, \textcite{e78e103697e2476a8074104b2fb45398}, \textcite{Arnaudon_2018}, and \textcite{Arnaudon_2022}, examine stochastic \emph{landmark shapes}. This work follows these, where landmarks are ordered sets $q = \set{x_i}_{i=1}^n$ of points in some underlying domain $\Omega$, here $\bb R^2$ or $\bb R^3$. The domain, along with landmarks within it, may be \emph{deformed} by diffeomorphic actions. See e.g.\ \textcite[chapter 4]{ff5d0f0a23c84a18a4cf60bdf8efe1aa} for details.
Endowing the underlying domain with a kernel allows construction of stochastic diffeomorphisms. Here, kernels of the type $K(x, y) = k(x, y) I_{d \times d}$ for a scalar kernel function $k$ are assumed.

Crucially, for landmark shapes, it suffices to evaluate these in the landmarks themselves.
\textcite[chapter 4.2]{alma99122437615805763} formalise this; ensuring that the landmarks move according to a stochastic flow of diffeomorphisms suffices. That may be ensured by the following proposition.

\begin{proposition}[Stochastic Flow of Landmarks]\label{prop:landmark_flow}
    Let $q = \set{x_i}_{i=1}^n \subseteq \Omega \subseteq \bb R^d$ be a landmark shape and $K$ be a kernel. Let $W_t$ be a collection of $n$ independent $d$-dimensional Wiener processes. Define the collection of stochastic processes
    \begin{equation}\label{eq:landmark_flow}
        \d X_t^i = \sum_{j=1}^n K(X_t^i, X_t^j) \d W_t^j, \quad X_s^i = x_i, t \geq s
    \end{equation}
    where superscripts index the processes; one for each landmark. Notice how each process operates on a single point in $\Omega$. As such, these may be lifted to a stochastic flow on $\Omega$.
\end{proposition}

This collection of processes may conveniently be stacked into a multivariate stochastic process, expressed as the \ito{} form stochastic differential equation
\begin{equation}\label{eq:vector_landmark_flow}
    \d X_{t,\ell} = \sigma(X_t) \d W_{t,\ell}
\end{equation}
with
\begin{equation}
    \sigma(X_t)^{i,j} = k(X_t^i, X_t^j)
\end{equation}
where the $\ell = 1,\dots,d$ subscript indicates the component of each landmark, while the superscripts index the landmarks (in the stacked representation). This is the process considered in this work, an instance of the \emph{Kunita flow} family of processes. Notice that it is a non-linear diffusion process, which for many (potentially several hundreds of) landmarks becomes rather high-dimensional.

This model is desirable, as it adheres to the reasonable assumption that changes occur according the \emph{current} state, rather than e.g.\ the initial state. Unfortunately, transition densities are unknown.

\section{DIFFUSION BRIDGES}
This work employs \emph{diffusion bridges} --- diffusion processes conditioned to hit a specific value at a given time --- as proposals in an importance sampler. To describe that process, as well as how to approximate such processes, some background will be briefly presented here.

\subsection{Bridges}
The fundamental tool to express arbitrary diffusion bridges is Doob's $h$-transform.

\begin{theorem}[Doob's $h$-transform] \label{thm:doobs-h-transform}
    Given a diffusion\footnote{Doob's $h$-transform applies in greater generality than what is necessary in this work.} process $\{X_t\}_{0 \leq t \leq T}$, a new diffusion process $\{X_t^\star\}_{0 \leq t \leq T}$, which is conditioned to hit some value $X_T^\star$ at time $T$, can be constructed, and has \ito{} form SDE
    \begin{equation} \label{eq:diffusion_bridge_general}
        \d X_t^\star = f^\star(X_t^\star) \d t + \sigma(X_t^\star) \d W_t
    \end{equation}
    where
    \begin{equation}
        f^\star(X_t^\star) = f(X_t^\star) + \Sigma(X_t^\star) \nabla \log p(X_T^\star \given X_t^\star)
    \end{equation}
    and
    \begin{equation} \label{eq:Sigma_def}
        \Sigma(X_t^\star) = \sigma(X_t^\star) \sigma(X_t^\star)^\top.
    \end{equation}
    Such diffusion processes conditioned to end at a specific value at a specific time point are called \emph{diffusion bridges}. Note that $p$ is the transition density of the original, unconditioned process $X_t$.
\end{theorem}

Refer to e.g.\ \textcite[chapter 7.5]{Sarkka_Solin_2019} for a proof.

This introduces a dependency on the logarithm of a transition density, which is generally --- and certainly in the context of Kunita flows --- unknown. Fortunately, \textcite{heng22} recently introduced a way to approximate the \emph{reverse time} diffusion bridge, which suffices for the methods presented in this work.

\begin{theorem}[Time Reversal] \label{thm:time-reversal}
    A diffusion process $\{X_t\}_{0 \leq t \leq T}$ has, under some mild regularity conditions, a \emph{reverse time} diffusion process $\{\bar X_t\}_{0 \leq t \leq T}$ given by the SDE
    \begin{equation} \label{eq:reverse_time}
        \d \bar X_t = \bar f(\bar X_t, t) \d t + \sigma(\bar X_t, t) \d \bar W_t
    \end{equation}
    with a new Wiener process $\{\bar W_t\}_{t \in [0, T]}$, the same diffusion term $\sigma$, and a new drift term
    \begin{equation} \label{eq:reverse_time_drift_alternative}
        \begin{split}
            \bar f(\bar X_t, t) = f(\bar X_t, t) &- \nabla \cdot \Sigma(\bar X_t, t) \\
            &- \Sigma(\bar X_t, t) \nabla \log p(\bar X_t \given \bar X_0).
        \end{split}
    \end{equation}
\end{theorem}

Note that the transition density $p$ is from the original, forward-time process, $X$. This result is due to \textcite{ANDERSON1982313}.

The crucial result comes from combining these theorems.

\begin{corollary}[Reverse Time Bridge Process] \label{cor:reverse_bridge}
    Given a diffusion process, $X_t$, admitting \ito{} form SDE
    \begin{equation}
        \d X_t = f(t, X_t) \d t + \sigma(t, X_t) \d W_t,
    \end{equation}
    applying Doob's $h$-transform (Theorem \ref{thm:doobs-h-transform}) and \citeauthor{ANDERSON1982313}'s time reversal (Theorem \ref{thm:time-reversal}) to $X_t$, in that order, yields a \emph{reverse time diffusion bridge}\footnotemark{} with \ito{} form SDE
    \begin{equation}
        \d \bar X_t^\star = \bar f^\star(\bar X_t^\star, t) \d t + \sigma(\bar X_t^\star, t) \d \bar W_t^\star
    \end{equation}
    with a new Wiener process $\bar W_t^\star$, the same diffusion term $\sigma$, and a new drift term
    \begin{equation}
        \begin{split}
            \bar f^\star(\bar X_t^\star, t) = f(\bar X_t^\star, t) &- \Sigma(\bar X_t^\star, t) \nabla \log p(\bar X_t^\star \given \bar X_0^\star) \\
            &- \nabla \cdot \Sigma(\bar X_t^\star, t).
        \end{split}
    \end{equation}
    The transition density, $p$, is that of the original, unconditioned process, $X_t$.
    \footnotetext{The order matters here; this is the time reversal of a diffusion bridge --- i.e., a diffusion bridge running backwards in time.}
\end{corollary}

Vitally, the gradient of log-transition density --- often denoted the \emph{score} ---  is the only unknown part here.

\subsection{Approximation}
\textcite{heng22} shows how the score may be approximated by \emph{score matching} using the objective function
\begin{equation}
    \frac{1}{2} \expectation_{\omega_t \sim \bb P^{x_0}} \left[ \int_0^T \| s_\phi(t, \omega_t) - s(t, \omega_t) \|_{\Sigma(t, \omega_t)}^2 \d t \right]
\end{equation}
where $s_\phi$ is the learned approximator for the true (unknown) score $s$. This expression may be approximated in the usual Monte Carlo fashion utilising time discretisations such as the Euler-Maruyama scheme.

With
\begin{align}
    p_i^j &= s_\phi(\tau_i, Y_{i+1}^j) \\
    v_i^j &= Y_{i+1}^j - Y_i^j - f(\tau_i, Y_i^j) \Delta t \\
    \Sigma_i^j &= \Delta t \Sigma(\tau_i, Y_i^j)
\end{align}
for notational brevity, the approximation becomes
\begin{equation}\label{eq:tractable_score_matching}
    \frac{1}{N} \sum_{j=1}^N \sum_{i=1}^M \Delta t \| p_i^j + {\Sigma_i^j}^{-1} v_i^j \|_{\Sigma_i^j}^2
\end{equation}
for $N$ Monte Carlo samples, $M$ time subdivisions, each of length $\Delta t$, for the Euler-Maruyama scheme, which also establishes $v_i^j$ and $\Sigma_i^j$.

This allows reverse time diffusion bridges (conditioned processes) to be approximated using sample paths of the unconditioned process obtainable via Euler-Maruyama simulation. Furthermore, owing to the Markov property of diffusion processes, each increment of such a sampled path may be utilised as i.i.d.\ samples to learn from. As the objective function is differentiable, modern deep learning may be employed as function approximators.

\subsection{Numerical Stability}\label{sec:bridge:stable}
\textcite{heng22} shows that the score matching objective function of \eqref{eq:tractable_score_matching} works well in many circumstances. However, for the morphometry processes examined in this work, where the covariance matrix $\Sigma$ is large and numerically problematically near singular, the dependence on its inverse is limiting. However, as the following result shows, the matrix inversions can be entirely circumvented.

\begin{theorem}[Numerically Stable Equivalent Objective Function]\label{thm:stable_loss}
    For vectors $p, v \in \bb R^d$, and a symmetric positive definite matrix $\Sigma \in \bb R^{d \times d}$, it holds that
    \begin{equation}
        \| p + \Sigma^{-1} v \|_\Sigma^2 = \|p\|_\Sigma^2 + 2 p^\top v + c
    \end{equation}
    where $c$ is some constant, which is independent of $p$.
\end{theorem}

\begin{proof}
    Recall that positive definiteness suffices for invertibility, and that inversion preserves symmetry. By definition of weighted norms and repeated use of standard linear algebra operations, it follows that
    \begin{align}
        &\| p + \Sigma^{-1} v \|_\Sigma^2 \\
        ={} &(p + \Sigma^{-1} v)^\top \Sigma (p + \Sigma^{-1} v) \\
        ={} &p^\top \Sigma (p + \Sigma^{-1} v) + (\Sigma^{-1} v)^\top \Sigma (p + \Sigma^{-1} v) \\
        \begin{split}
        ={} &p^\top \Sigma p + p^\top \Sigma \Sigma^{-1} v + (\Sigma^{-1} v)^\top \Sigma p \\
            &\qquad + (\Sigma^{-1} v)^\top \Sigma \Sigma^{-1} v
        \end{split} \\
        ={} &p^\top \Sigma p + p^\top v + v^\top \Sigma^{-\top} \Sigma p + v^\top \Sigma^{-1} v \\
        ={} &p^\top \Sigma p + p^\top v + v^\top p + v^\top \Sigma^{-1} v \\
        ={} &\|p\|_\Sigma^2 + 2 p^\top v + \|v\|_{\Sigma^{-1}}^2
    \end{align}
    where the only term involving $\Sigma^{-1}$ crucially does not involve $p$; hence the claim.
\end{proof}

Thus, courtesy of Corollary~\ref{cor:reverse_bridge} and Theorem~\ref{thm:stable_loss}, using the objective function
\begin{equation}
    \frac{1}{N} \sum_{j=1}^N \sum_{i=1}^M \Delta t \big( \|p_i^j\|_{\Sigma_i^j}^2 + 2 {p_i^j}^\top v_i^j \big)
\end{equation}
in a gradient descent-based score matching routine will provide an approximator that may be used to simulate reverse time diffusion bridges.

\section{LIKELIHOOD ESTIMATION}
Because the transition density $p(X_{t_1} \given X_{t_0})$ is unavailable in closed form, we cannot directly estimate parameters using the likelihood. \textcite{500d2f1b-d6f7-3edc-8f25-66b3ad9b7d1d} introduce \emph{simulated likelihood estimation}, one such approximation scheme, complete with asymptotic  consistency results (\cite{ae19f427-d110-3aee-8cbb-d096fa1e809e}). Inspired by \textcite{doi:10.1198/jasa.2010.tm09057}, this paper proposes an extension of simulated likelihood estimation using diffusion bridges as proposals in an importance sampler.

\begin{proposition}[Importance Sampled Simulated Likelihood Estimation]\label{prop:importance_sml}
    For a diffusion process
    \begin{equation}
        \d X_t = f(X_t) \d t + \sigma(X_t) \d W_t, \quad X_{t_0} = x_{t_0}
    \end{equation}
    let
    \begin{equation}
        \d X_t^\star = f^\star(X_t^\star) \d t + \sigma^\star(X_t^\star) \d W_t^\star, \quad X_{t_0}^\star = X_{t_0}
    \end{equation}
    be the diffusion bridge process conditioned to end at $X_{t_1}$ at time $t_1$, having (transition) density $p^\star$. For a sequence of time steps
    \begin{equation}
        t_0 = \tau_0 < \tau_1 < \dots < \tau_{M-1} = t_z < \tau_M = t_1
    \end{equation}
    the transition density $p(X_{t_1} \given X_{t_0})$ obeys the identity
    \begin{equation}\label{eq:importance_sml}
        \begin{split}
            &p(X_{t_1} \given X_{t_0}) \\
            &=\expectation \left[ p(X_{\tau_M} \given X_{\tau_{M-1}}) \frac{\displaystyle \prod_{i=1}^{M-1} p(X_{\tau_i} \given X_{\tau_{i-1}})}{\displaystyle \prod_{i=1}^{M-1} p^\star(X_{\tau_i} \given X_{\tau_{i-1}})} \right]
        \end{split}
    \end{equation}
    where the expectation is with regards to $(X_{\tau_1}, \dots, X_{\tau_{M-1}}) \sim p^\star(\cdot \given X_{\tau_0})$.
\end{proposition}

Assuming access to a sampler of $p^\star$, this quantity may be estimated in the usual Monte Carlo fashion with
\begin{align}
    &p(X_{\tau_i} \given X_{\tau_{i-1}}) \\
    \approx{} &\tilde p(X_{\tau_i} \given X_{\tau_{i-1}}) \\
    ={} &\c N\big(X_{\tau_i};\; X_{\tau_{i-1}} + f(X_{\tau_{i-1}}) \Delta_i, \Sigma(X_{\tau_{i-1}}) \Delta_i \big) \\\\
    &p^\star(X_{\tau_i} \given X_{\tau_{i-1}}) \\
    \approx{} &\tilde p^\star(X_{\tau_i} \given X_{\tau_{i-1}}) \\
    ={} &\c N\big(X_{\tau_i};\; X_{\tau_{i-1}} + f^\star(X_{\tau_{i-1}}) \Delta_i, \Sigma^\star(X_{\tau_{i-1}}) \Delta_i \big)
\end{align}
where $\Delta_i = (\tau_i - \tau_{i-1})$,
which, by arguments similar to the Euler-Maruyama simulation scheme, become reasonable approximations when $\tau_i - \tau_{i-1}$ are small; equivalently, when $M$ is large.

\begin{proof}
    Following standard importance sampling estimator methodology, it holds that
    \begin{align}
        &p(X_{t_1} \given X_{t_0}) \\
        \eq={} &\int p(X_{\tau_1}, X_{\tau_2}, \dots, \underbrace{X_{\tau_M}}_{=X_{t_1}} \given \underbrace{X_{\tau_0}}_{=X_{t_0}}) \d (X_{\tau_1}, \dots, X_{\tau_{M-1}}) \\
        \begin{split}
        \eq={} &\int p(X_{\tau_M} \given X_{\tau_{M-1}}) \prod_{i=1}^{M-1} p(X_{\tau_i} \given X_{\tau_{i-1}}) \\
            &\qquad \d (X_{\tau_1}, \dots, X_{\tau_{M-1}})
        \end{split} \\
        \begin{split}
        \eq={} &\int p(X_{\tau_M} \given X_{\tau_{M-1}}) \prod_{i=1}^{M-1} p(X_{\tau_i} \given X_{\tau_{i-1}}) \\
            &\qquad \frac{\displaystyle \prod_{i=1}^{M-1} p^\star(X_{\tau_i} \given X_{\tau_{i-1}})}{\displaystyle \prod_{i=1}^{M-1} p^\star(X_{\tau_i} \given X_{\tau_{i-1}})} \d (X_{\tau_1}, \dots, X_{\tau_{M-1}})
        \end{split} \\
        \begin{split}
        \eq={} &\int p(X_{\tau_M} \given X_{\tau_{M-1}}) \frac{\displaystyle \prod_{i=1}^{M-1} p(X_{\tau_i} \given X_{\tau_{i-1}})}{\displaystyle \prod_{i=1}^{M-1} p^\star(X_{\tau_i} \given X_{\tau_{i-1}})} \\
            &\qquad p^\star(X_{\tau_1}, \dots, X_{\tau_{M-1}} \given X_{\tau_0}) \d (X_{\tau_1}, \dots, X_{\tau_{M-1}})
        \end{split} \\
        \eq={} &\expectation \left[ p(X_{\tau_M} \given X_{\tau_{M-1}}) \frac{\displaystyle \prod_{i=1}^{M-1} p(X_{\tau_i} \given X_{\tau_{i-1}})}{\displaystyle \prod_{i=1}^{M-1} p^\star(X_{\tau_i} \given X_{\tau_{i-1}})} \right]
    \end{align}
    where
    \begin{enumerate}
        \item is by the Chapman-Kolmogorov equations,
        \item utilises Markov factorisation,
        \item multiplies by the neutral element, introducing new densities following standard importance sampling practice,
        \item rearranges terms and utilises Markov factorisation (in reverse), and
        \item recognises the expectation with regards to $(X_{\tau_1}, \dots, X_{\tau_{M-1}}) \sim p^\star(\cdot \given X_{\tau_0})$.
    \end{enumerate}
\end{proof}

\subsection{Numerical Stability}\label{sec:estimation:stable}
While theoretically sound, this approach contains multiple numerical challenges, particularly for large $M$; increasing $M$ implies decreasing $\Delta_i$, which makes inverting $\Sigma(X_{\tau_i}) \Delta_i$ even more troublesome. Furthermore, products of many small numbers may be problematic, and division of two such small numbers even more so. Fortunately, these numerical instability problems can largely be alleviated or entirely circumvented.

For notational conciseness, let
\begin{align}
    \mu_{\tau_i} &= X_{\tau_{i-1}} + f(X_{\tau_{i-1}}) (\tau_i - \tau_{i-1}) \\
    \mu_{\tau_i}^\star &= X_{\tau_{i-1}} + f^\star(X_{\tau_{i-1}}) (\tau_i - \tau_{i-1}) \\
    \Sigma_{\tau_i} &= \Sigma(X_{\tau_{i-1}}) (\tau_i - \tau_{i-1}) \\
    \Sigma_{\tau_i}^\star &= \Sigma^\star(X_{\tau_{i-1}}) (\tau_i - \tau_{i-1}),
\end{align}
recall that $\Sigma(X_{\tau_i}) = \Sigma^\star(X_{\tau_i})$ for the diffusion bridges in question, and consider the Euler-Maruyama approximation of the logarithm of the fraction in \eqref{eq:importance_sml}
\begin{equation}\label{eq:stable_importance_weight}
    \begin{split}
        &\log \frac{\displaystyle \prod_{i=1}^{M-1} \tilde p(X_{\tau_i} \given X_{\tau_{i-1}})}{\displaystyle \prod_{i=1}^{M-1} \tilde p^\star(X_{\tau_i} \given X_{\tau_{i-1}})} \\
        ={} &\frac{1}{2} \sum_{i=1}^{M-1} \Big( (X_{\tau_i} - \mu^\star)^\top {\Sigma_{\tau_{i-1}}^\star}^{-1} (X_{\tau_i} - \mu^\star) \\
        &\qquad\qquad -(X_{\tau_i} - \mu)^\top {\Sigma_{\tau_{i-1}}}^{-1} (X_{\tau_i} - \mu) \Big).
    \end{split}
\end{equation}
This removed problematic determinant computations, along with divisions of tiny products. Then utilising the numerically stable log-sum-exp trick to approximate $p(X_{t_1} \given X_{t_0})$ by
\begin{equation}
    \logsumexp \big(w^i + \log \tilde p(X_{\tau_M} \given X_{\tau_{M-1}}^i)\big) - \log N
\end{equation}
$i = 1, \dots, N$, with superscripts indicating sample indices, and each $w^i$ computed as in \eqref{eq:stable_importance_weight}, helps.

Furthermore, expressing $\Sigma(X_{\tau_i})$ as $\sigma(X_{\tau_i}) \sigma(X_{\tau_i})^\top$ reveals that
\begin{align}
    &(X_{\tau_i} - \mu_{\tau_i})^\top \big(\sigma_{\tau_i} \sigma_{\tau_i}^\top\big)^{-1} (X_{\tau_i} - \mu_{\tau_i}) \\
    ={} &(X_{\tau_i} - \mu_{\tau_i})^\top \sigma_{\tau_i}^{-\top} \sigma_{\tau_i}^{-1} (X_{\tau_i} - \mu_{\tau_i}) \\
    ={} &\big(\sigma_{\tau_i}^{-1} (X_{\tau_i} - \mu_{\tau_i})\big)^\top \sigma_{\tau_i}^{-1} (X_{\tau_i} - \mu_{\tau_i}) \\
    ={} &z^\top z
\end{align}
with $z = \sigma_{\tau_i}^{-1} (X_{\tau_i} - \mu_{\tau_i})$. This quantity may routinely be found by solving the linear equation
\begin{align}
    \sigma_{\tau_i} z = X_{\tau_i} - \mu_{\tau_i}
\end{align}
for $z$. For greater yet numerical stability, approximate it by a least squares solver.

If applying likelihood estimation for parameter inference, it suffices to be within a constant factor of the likelihood (a constant additive constant off of the log-likelihood, respectively). For variance parameter inference, in particular, this observation may be utilised to further stabilise computations. Recall the log-density of a $k$-dimensional multivariate Gaussian distribution
\begin{equation}
    -\frac{k}{2} \log(2 \pi) - \frac{1}{2} \log \det \Sigma - \frac{1}{2} (x - \mu)^\top \Sigma^{-1} (x - \mu)
\end{equation}
and assume $\Sigma = v \sigma \sigma^\top$, where $\sigma$ is given by some covariance structure and $v$ is the parameter of interest. Then
\begin{align}
    \log \det \Sigma
    &= \log \det (v \sigma \sigma^\top) \\
    &= k \log v + \log \det (\sigma \sigma^\top)
\end{align}
where the latter term is constant in $v$, and
\begin{align}
    &(x - \mu)^\top \Sigma^{-1} (x - \mu) \\
    ={} &(x - \mu)^\top (v \sigma \sigma^\top)^{-1} (x - \mu) \\
    ={} & \frac{1}{v} \big(\sigma^{-1} (x - \mu)\big)^\top \sigma^{-1} (x - \mu) \\
    ={} & \frac{1}{v} z^\top z
\end{align}
which may be computed stably as previously described. Combined,
\begin{equation}
    -\frac{k}{2} \log v - \frac{1}{v} z^\top z
\end{equation}
is thus a constant term off of the log-likelihood in a variance parameter inference search. Notice how this expression involves neither determinant computations nor matrix inversions. And when using this in the log-sum-exp trick expression of the proposed importance sampler, neither will there be any products or divisions of tiny numbers. This range of tricks transforms the importance sampler into a useful estimation scheme for otherwise completely numerically intractable systems.

\section{INFERENCE}
The importance sampler introduced in the previous section depends on sampled diffusion bridges. Knowing the drift and diffusion terms of the conditioned SDEs, such samples can be obtained by numerical solvers. If careful, the entire importance sampling likelihood estimator can be constructed in a differentiable manner. While sampling is non-differentiable, a trick akin to the reparameterisation trick in the variational auto-encoder literature can be employed to nonetheless make the samples differentiable.

As $\sigma \epsilon + \mu \sim \c N(\mu, \sigma \sigma^\top)$ for $\epsilon \sim \c N(0, I)$, an Euler-Maruyama sample path may be obtained by the simple algorithm
\begin{algorithmic}
    \State $\Delta t \gets (t_1 - t_0) / M$
    \State $ws \gets [\c N(0, \sqrt{\Delta t}\, I)$ for $i = 1, \dots, M]$
    \State $ys \gets \texttt{scan}((y, w) \mapsto y + f(y) \Delta t + \sigma(y) w,\; y_0,\; ws)$
\end{algorithmic}
where \texttt{scan} is in the Haskell or \texttt{jax.lax} sense (returns successively reduced values). This sampling scheme additionally proves rather efficient; written in \texttt{jit}'ed \texttt{jax}, it is even faster than e.g.\ the excellent \texttt{diffrax.diffeqsolve} (\cite{kidger2021on}).

\subsection{Variance Parameter Inference}

Viability of parameter inference using the proposed likelihood estimation procedure is illustrated here. For a process mimicking the desired Kunita flows, but for which analytical transition densities are known, consider fixing
\begin{equation}\label{eq:brownian}
    \sigma(X_{t})^{i,j} = k(X_{t_0}^{i,j}, X_{t_0}^{i,j})
\end{equation}
to use the \emph{initial} positions of the landmarks, which corresponds to a Brownian motion with known Gaussian transition densities.

For a challenging baseline, consider 100 landmark outline discretisations of the two butterflies illustrated in Figure~\ref{fig:butterfly_landmarks} as initial ($X_{t_0}, t_0 = 0$) and terminal ($X_{t_1}, t_1 = 1$) shapes in a conditioned zero-drift diffusion process governed by diffusion term given by \eqref{eq:brownian}.


Figure~\ref{fig:baseline:variance} illustrates log-likelihood curves for the variance parameter for this conditioned diffusion process computed by four different methods.

\begin{figure}[t]
    \input{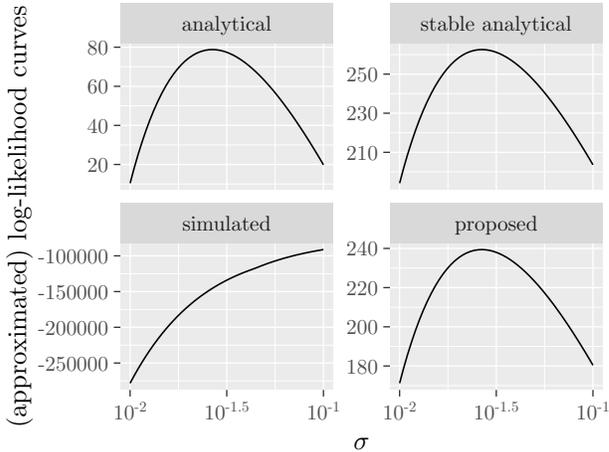}
    \caption{Log-likelihood curves for the conditioned diffusion process evolving landmarks of \emph{Papilio ambrax} into \emph{Papilio slateri}. Method `analytical' is the true log-likelihood, known in this simplified process; `stable analytical' uses the stable but off-by-a-constant computation presented in Section~\ref{sec:estimation:stable}; `simulated' is computed by the simulated likelihood estimation method of \textcite{500d2f1b-d6f7-3edc-8f25-66b3ad9b7d1d}; `proposed' uses the stable importance sampler proposed in this work. The latter two methods use 1000 Monte Carlo samples with 1000 simulation time steps. Notice how the off-by-a-constant proposed method exactly captures the shape of the log-likelihood curve, allowing parameter inference.}
    \label{fig:baseline:variance}
\end{figure}

\subsection{Diffusion Mean}\label{sec:diffusion_mean}
As mentioned, the estimation procedure is fully differentiable. Not only does this allow efficient parameter search, it also allows establishing \emph{diffusion means} (\cite{10.3150/22-BEJ1578}); given a specified diffusion process and a collection of observations, finding the most likely initial points of the process. Choose an arbitrary point, consider the diffusion bridges from that point to each of the observations, and update the diffusion mean estimate by gradient ascent on the sum of the log-likelihoods of those bridges.

Figure~\ref{fig:brownian:diffusion_mean} illustrates diffusion mean estimation. Ten samples from a two-dimensional unit variance zero covariance two-dimensional Brownian motion are sampled. An arbitrary initial guess of the diffusion mean is chosen (red circle), which is subsequently moved by gradient ascent on the joint log-likelihood estimate of the diffusion bridges. The estimate matches the true diffusion mean --- which for Brownian motions is known --- almost perfectly.

\begin{figure}[H]
    \includegraphics[width=\linewidth]{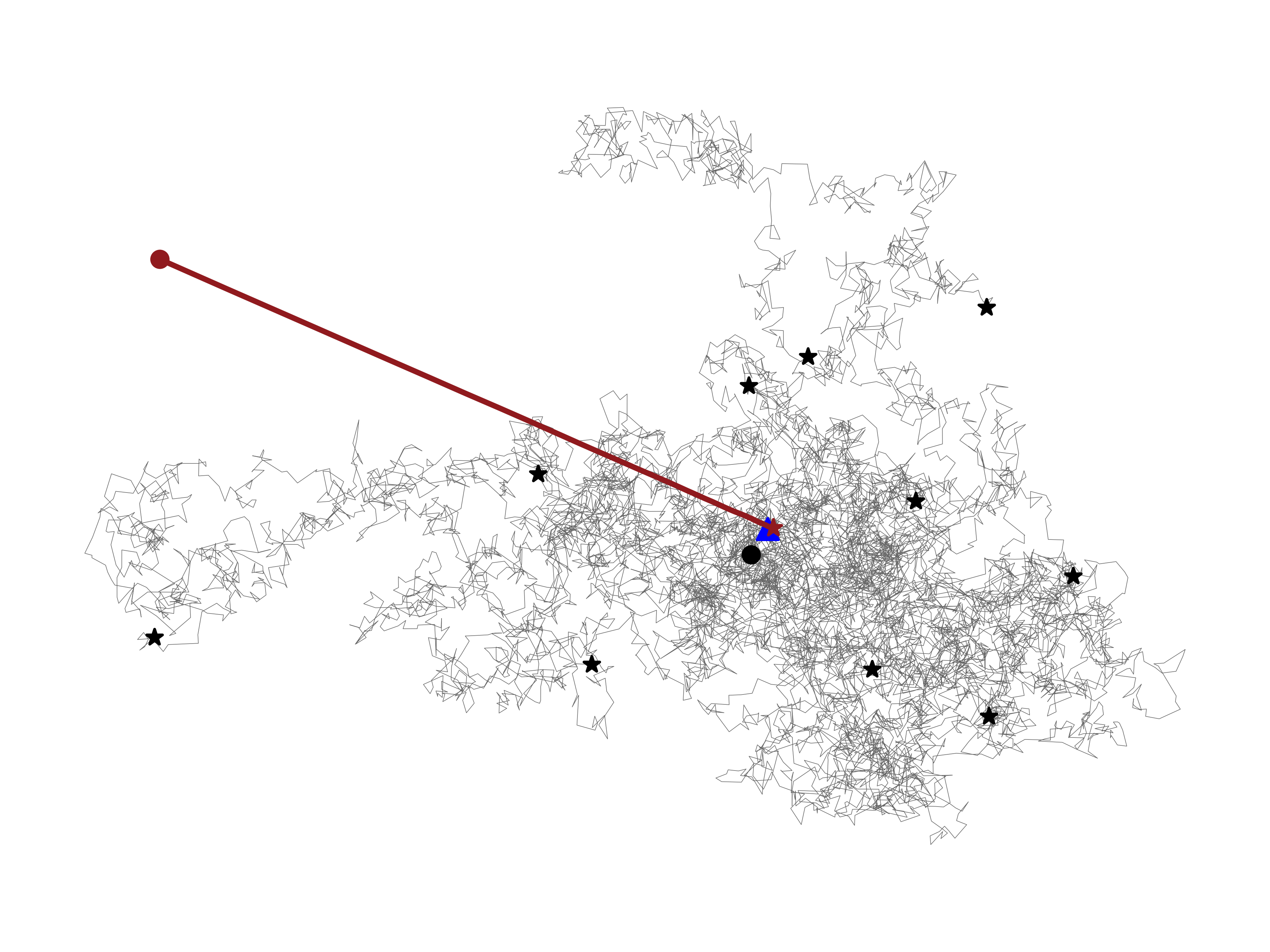}
    \caption{Black stars illustrate observations of sampled Brownian motions initiated at the black circle at the origin. Red circle illustrates initial diffusion mean estimate (chosen arbitrarily). The joint log-likelihood of diffusion bridges from the current diffusion mean estimate to each of the observations is computed and the diffusion mean estimate is updated by gradient ascent on it. The red line shows the path taken by the diffusion mean estimate. Red star illustrates final diffusion mean estimate, which coincides almost fully with the true diffusion mean, indicated by the blue triangle.}
    \label{fig:brownian:diffusion_mean}
\end{figure}

\section{APPLICATIONS}
Two distinct biological applications are presented to exemplify the utility of the proposed methodology. Whereas the previous section, for demonstration purposes, utilised a simplified diffusion process --- a Brownian motion --- for which transition densities and the score are analytically known, this section uses the actual process of interest; Kunita flows.

A neural network is trained using the score matching technique presented in Section~\ref{sec:bridge:stable} to serve as a function approximator of the unknown score. This score approximation is then used in the proposed stable importance sampler.

The structure of the neural network is illustrated in Figure~\ref{fig:network:inverse_u}, exemplified here for handling 100 2-dimensional landmarks for e.g.\ the butterflies. One additional input component is added for the time step.

\begin{figure}[h]
    \centering
    \includegraphics[width=\linewidth]{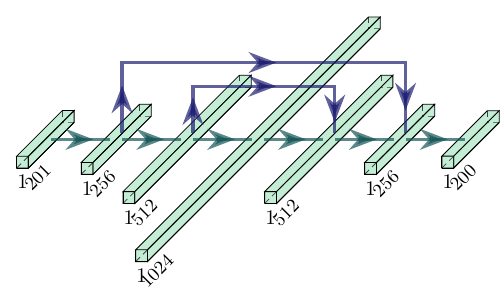}
    \caption{Structure of neural network used, shown here corresponding to 100-point two-dimensional landmark shape; the shape is flattened into a 200-dimensional vector and the time point is concatenated. The dark blue skip connections add the values element-wise to the later layers. These skip connections are found to help the model express the score fields properly.}
    \label{fig:network:inverse_u}
\end{figure}

To make a single network able to approximate scores for a range of variance parameters, the sinusoidal embedding of \textcite{NIPS2017_3f5ee243} is used for normalised log-variance parameters and included as a scale-shift operation on the down-sizing layers.

\subsection{Relationships}
The most likely variance parameter estimate for a diffusion bridge between two observations may informally be viewed as an indicator of similarity between the observations. Given e.g.\ archaeological observations, this similarity indication may provide pointers towards establishing relationships.

To demonstrate this, three observations of parietal bone outlines in canid skulls of known origin are considered; two from distinct specimens of \emph{Canis lupus} (grey wolf) and one from a specimen of \emph{Vulpes vulpes} (red fox). The outline of the parietal bone is chosen for its importance in skull structure and brain protection. Figure~\ref{fig:canidae_bridges} illustrates two of these landmark configurations, along with learned diffusion bridges connecting them with different variance parameters for the assumed Kunita flow.

\begin{figure}[h]
    \centering
    \begin{subfigure}[b]{0.5\linewidth}
        \centering
        \includegraphics[width=\linewidth]{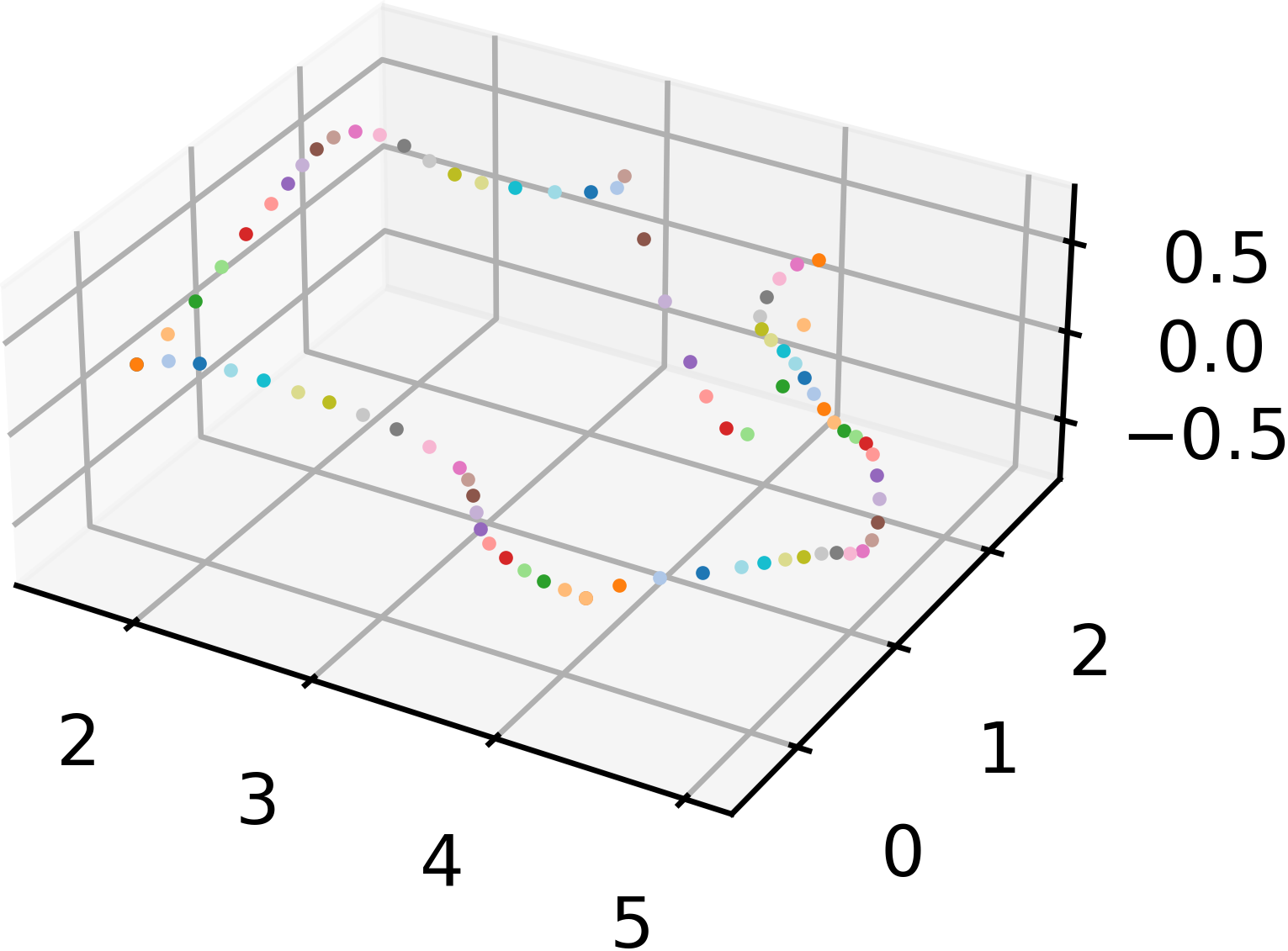}
        \caption{\emph{Canis lupus}}
    \end{subfigure}%
    \begin{subfigure}[b]{0.5\linewidth}
        \centering
        \includegraphics[width=\linewidth]{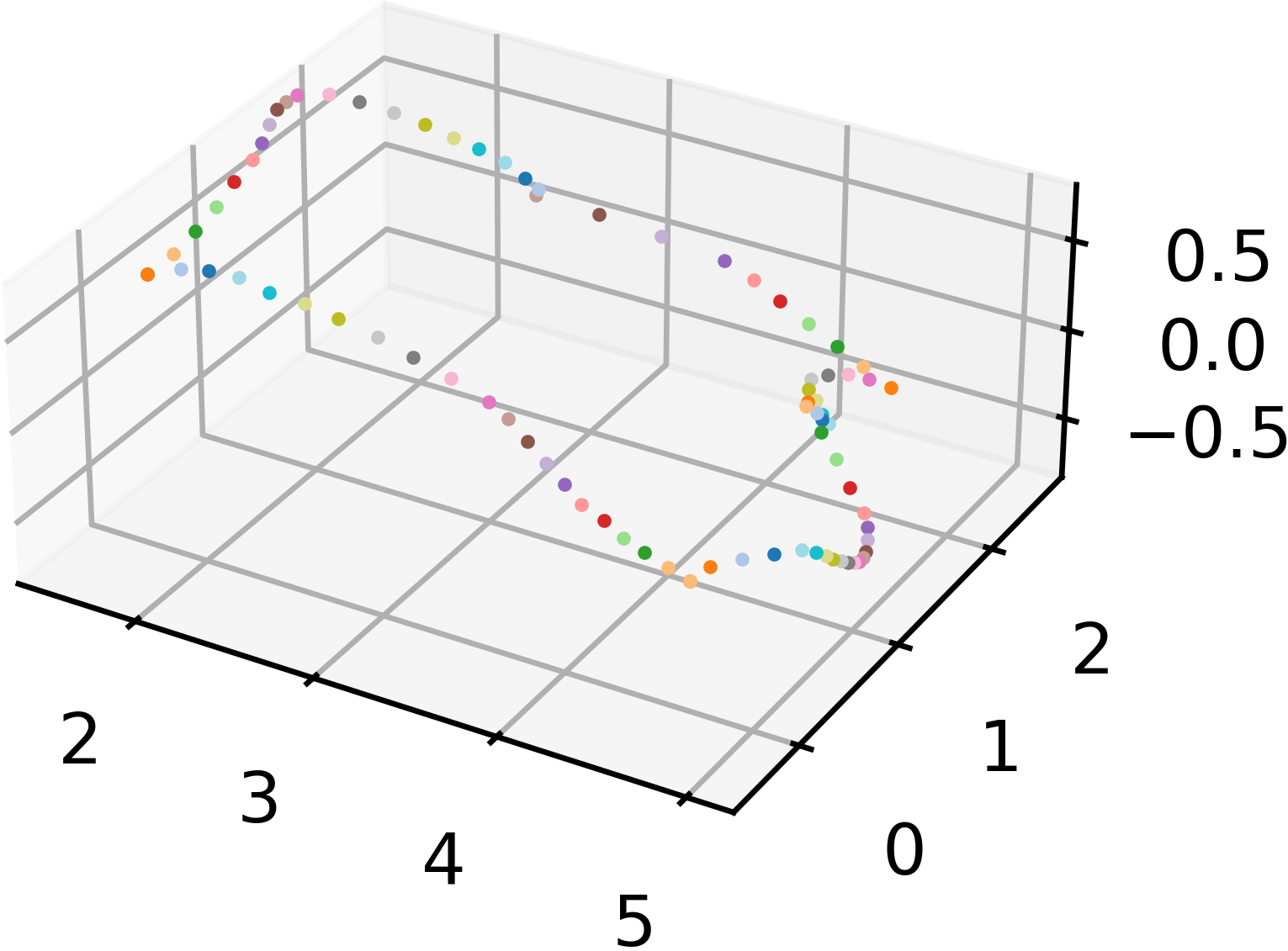}
        \caption{\emph{Vulpes vulpes}}
    \end{subfigure}
    \begin{subfigure}[b]{0.5\linewidth}
        \centering
        \includegraphics[width=\linewidth]{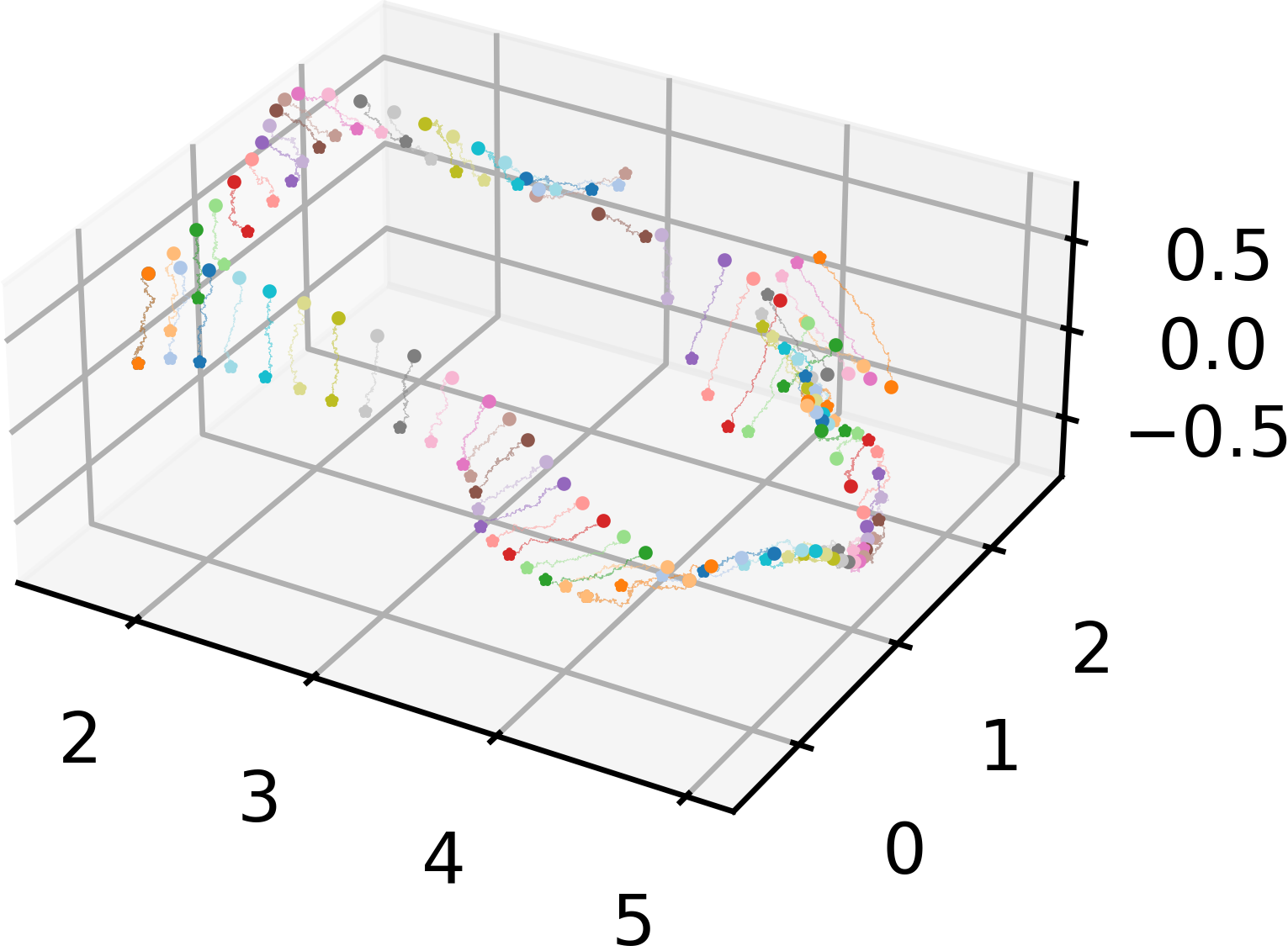}
        \caption{bridge, $\sigma = 0.0464$}
    \end{subfigure}%
    \begin{subfigure}[b]{0.5\linewidth}
        \centering
        \includegraphics[width=\linewidth]{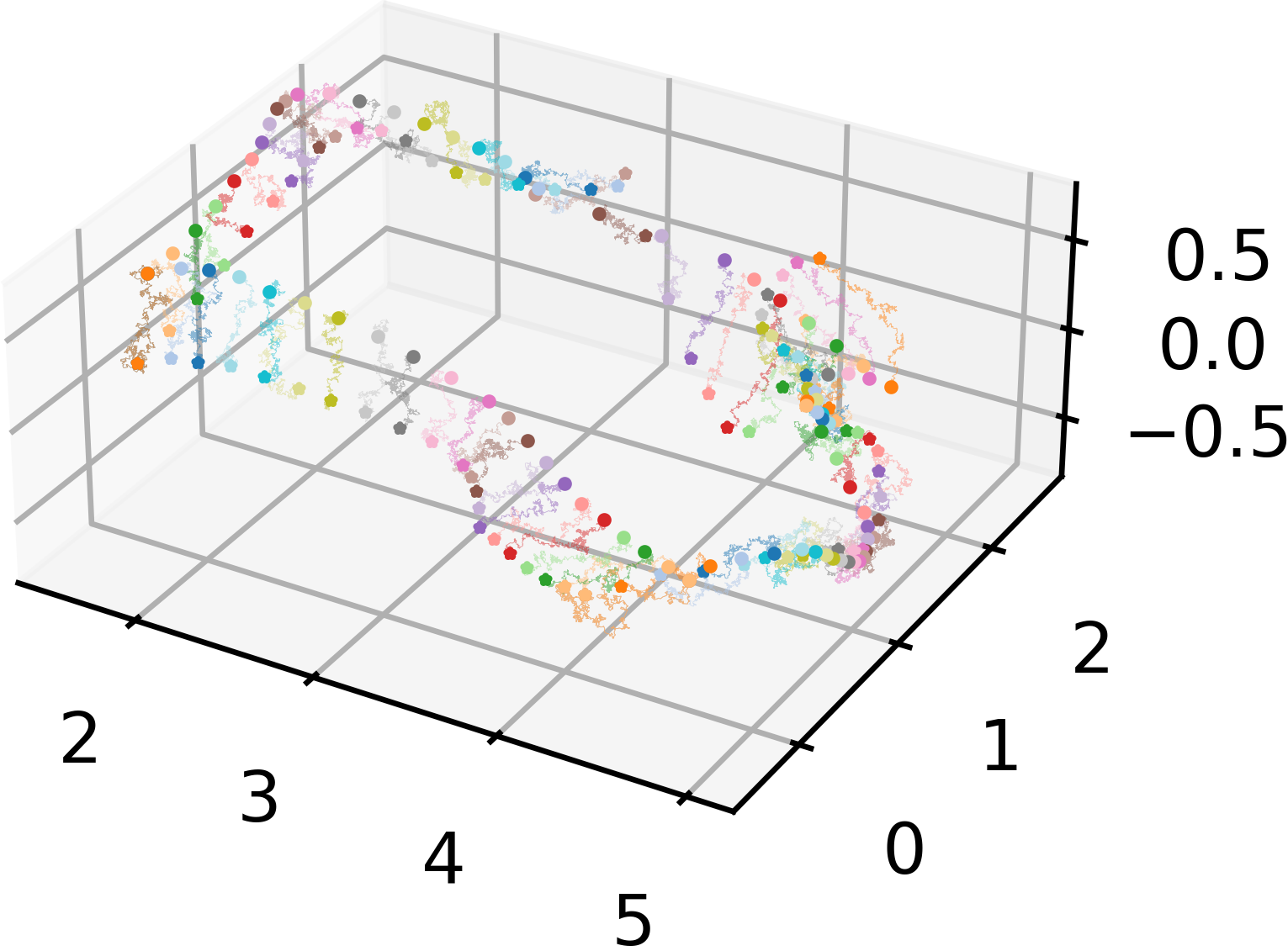}
        \caption{bridge, $\sigma = 0.129$}
    \end{subfigure}
    \caption{(a) and (b) illustrate landmark discretisations of the outline of the parietal bone of \emph{Canis lupus} and \emph{Vulpes vulpes} specimens, respectively. (c) and (d) show learned Kunita flow diffusion bridges between the landmark configurations of (a) and (b), using two different variance parameter values for the process. Data from \textcite{boyer2016morphosource}.}
    \label{fig:canidae_bridges}
\end{figure}

Variance parameters are inferred for the diffusion bridges between the two wolves and between one wolf and the fox, respectively. Figure~\ref{fig:canidae_ll} illustrates the approximated (off-by-a-constant) log-likelihood curves, with most likely value indicated. Unsurprisingly, the inferred most likely variance parameter is larger for the inter-species bridges than for intra-species bridges --- even after performing Procrustes alignment, eliminating size information.

\begin{figure}[H]
    \input{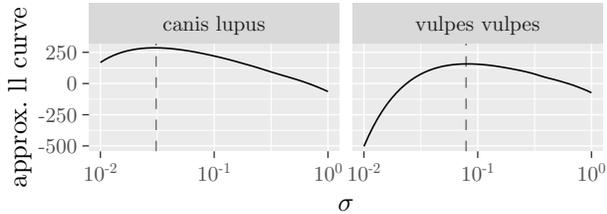}
    \caption{Approximated (off-by-a-constant) log-likelihood curves for the bridges illustrated in Figure~\ref{fig:canidae_bridges}. Dashed vertical lines indicate most likely variance parameter value for each of the bridges, computed using learned score approximators in the proposed importance sampler.}
    \label{fig:canidae_ll}
\end{figure}

Although expected, this result may serve as support for the proposed methodology, exactly because of how biologically obvious it is.

\subsection{Ancestral State}
Section~\ref{sec:diffusion_mean} established estimation of diffusion means using the proposed methodology. For a collection of observations of extant species, their diffusion mean\footnote{Here conceptually generalised to Kunita flows, despite being introduced specifically for Brownian motions in \textcite{10.3150/22-BEJ1578}} may serve as a most likely ancestral state candidate. Conditioning the neural network score approximator on initial state, and training on a distribution of initial states surrounding the collection of observations may allow such ancestral state reconstruction.

To illustrate this, six butterflies of the \emph{Papilio} genus are chosen, with and without swallow tails. One of them is chosen arbitrarily as an initial diffusion mean estimate, and this estimate is then deformed by gradient ascent on the joint log-likelihood of the bridges from it to each observation. Figure~\ref{fig:butterfly_mean_estimation} illustrates the progression of the mean estimate. Curiously, it moves towards the ancestral shape having swallow tails.


\begin{figure}[H]
    \centering
    \begin{subfigure}[b]{0.5\linewidth}
        \centering
        \includegraphics[width=\linewidth]{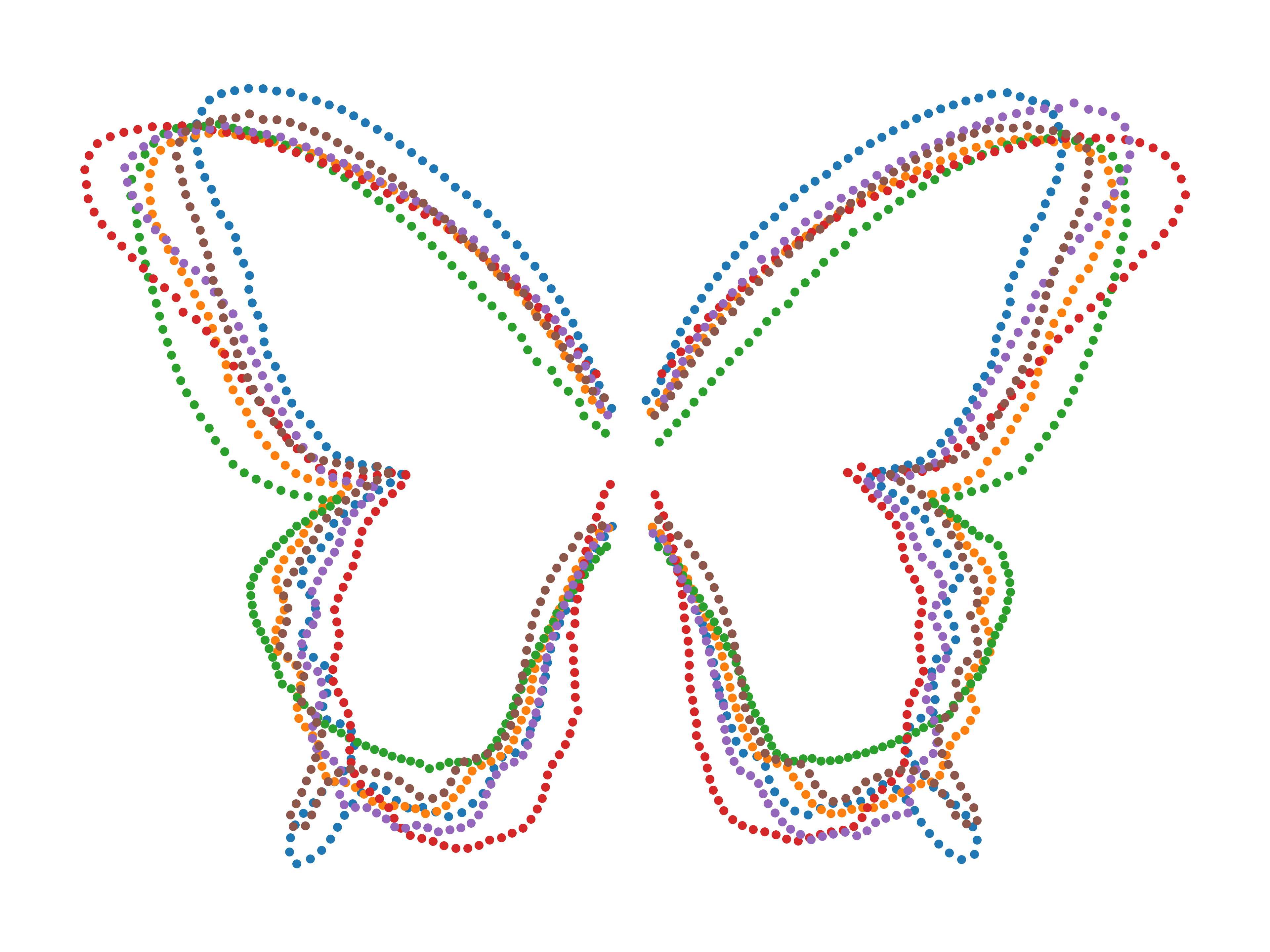}
        \caption{observations}
    \end{subfigure}%
    \begin{subfigure}[b]{0.5\linewidth}
        \centering
        \includegraphics[width=\linewidth]{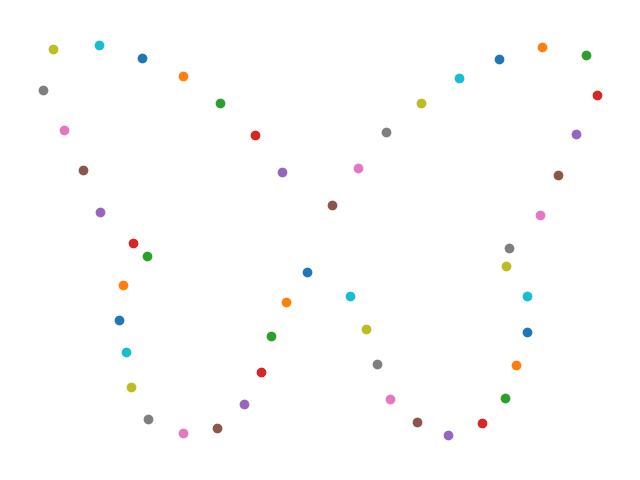}
        \caption{initial estimate}
    \end{subfigure}
    \begin{subfigure}[b]{0.5\linewidth}
        \centering
        \includegraphics[width=\linewidth]{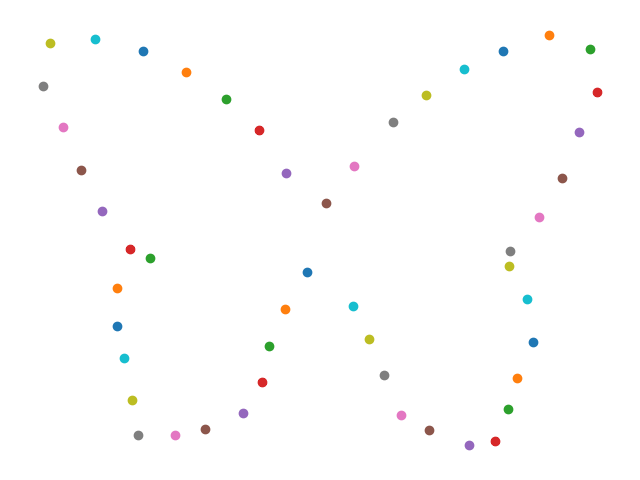}
        \caption{after 150 steps}
    \end{subfigure}%
    \begin{subfigure}[b]{0.5\linewidth}
        \centering
        \includegraphics[width=\linewidth]{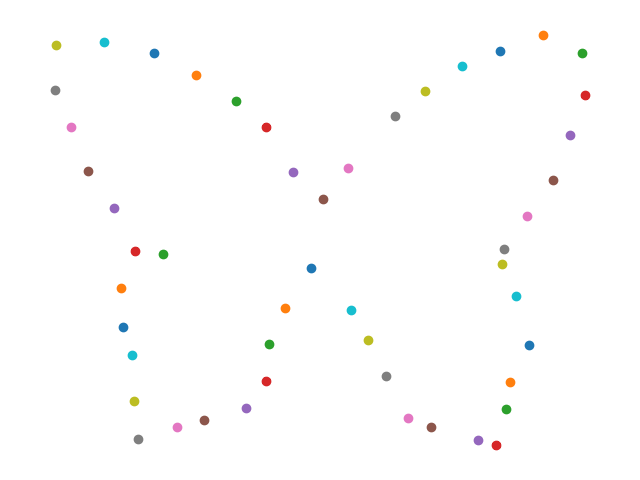}
        \caption{after 300 steps}
    \end{subfigure}
    \caption{Ancestral shape reconstruction of six butterflies of the \emph{Papilio} genus, shown in (a). One of them is chosen as an initial guess of their diffusion mean estimate, shown in (b). This estimate is updated via gradient ascent on the approximated log-likelihood of the diffusion bridges between the estimate and each of the observations. Plots (c) and (d) illustrate the progress of the diffusion mean estimate as an ancestral state reconstruction estimate. Notice how the ancestral state tends towards having swallow tails.}
    \label{fig:butterfly_mean_estimation}
\end{figure}

\section{CONCLUSION}
A differentiable likelihood estimation methodology capable of parameter inference and diffusion mean estimation in challenging, non-linear and high-dimensional diffusion processes has been presented. Powered by a symbiosis between modern deep learning --- score matching via a novel numerically stable objective function --- and theoretical statistics, it has been demonstrated how the proposed methodology is capable of delivering insights in the field of evolutionary biology.

We envision future work could investigate integrating this methodology on phylogenetic trees, as well as making neural networks conditioned on initial state more stable for better diffusion mean estimation.

\clearpage

\subsubsection*{Acknowledgements}
The work presented in this article was done at the Center for Computational Evolutionary Morphometry and is partially supported by the Novo Nordisk Foundation grant NNF18OC0052000 as well as a research grant (VIL40582) from VILLUM FONDEN and UCPH Data+ Strategy 2023 funds for interdisciplinary research.

\subsubsection*{References}

\printbibliography[heading=none]

\paragraph{Data}
Exact sources of the data used in this paper are presented in the table on the following page, which includes source of each species and where it was collected. 

The canid skulls from Bergen were lent by Dr. Hanneke J.M. Meijer,
Associate Professor \& Curator of Osteology,
University Museum of Bergen,
Hanneke.Meijer@uib.no

\onecolumn
\newpage

    \begin{sidewaystable}
    \small
\begin{tabular}{l|lllllll}
    \textbf{type} & \textbf{Class} &\textbf{scientificName}                        & \textbf{locality}                    &                      \textbf{institutionCode} & \textbf{ID} & Sex & Source  \\ \hline \hline
 2d image &Insecta &Papilio ambrax   Boisduval, 1832               & 13 km W of   Kennedy                     &    MCZ    & 89060         & male & Gbif, \cite{gbif_m} \\
 2d image &Insecta &Papilio deiphobus Linnaeus, 1758               & Seram {[}Ceram{]}                         &    MCZ    & 211983        & male & Gbif, \cite{gbif_m}  \\
 2d image &Insecta &Papilio polyxenes asterius Stoll, 1782         & Weston                                    &    MCZ    & 174079        & male & Gbif, \cite{gbif_m}  \\
 2d image &Insecta &Papilio   protenor Cramer, 1775                &    -  &    MCZ    & 180722        & male & Gbif, \cite{gbif_m}  \\
 2d image &Insecta &Papilio polytes Linnaeus, 1758                 & Bacan Batjan, Batchian             &    MCZ    & 170907        & male & Gbif, \cite{gbif_m}  \\
 2d image &Insecta &Papilio slateri                                & North Borneo                              &    MCZ    & 176631        & male & Gbif, \cite{gbif_m}  \\ \hline
 3d scan &Mammal &Canis lupus Linnaeus, 1758                     &  Bergen                                   &    UIB       &  B2           &  -   &   EvoMorphoLab    \\
 3d scan &Mammal &Canis lupus Linnaeus, 1758                     &  Bergen                                   &    UIB       &  2698         & -    &     EvoMorphoLab     \\
 3d scan &Mammal &Vulpes vulpes Linnaeus, 1758                   &   -                                         &    L-ahr  & l-ahr:208049& &  Morphosource, \cite{boyer2016morphosource} \\
\end{tabular}
\caption{Source of each species. L-ahr: Laboratory of Adam Hartstone-Rose, MCZ: Havard Museum of Comparative Zoology, UIB: Bergen University Natural History Museum}
\end{sidewaystable}

\end{document}